\documentclass{article}

\usepackage{url}            
\usepackage{microtype}
\usepackage{graphicx}
\usepackage{subcaption}
\usepackage{booktabs} 
\usepackage{wrapfig}

\usepackage{url}  
\usepackage[hidelinks,colorlinks=true,linkcolor=blue,citecolor=blue]{hyperref}

\usepackage{tikz}
\newcommand{\cblock}[3]{
  \hspace{-1.5mm}
  \begin{tikzpicture}
    [
    node/.style={square, minimum size=10mm, thick, line width=0pt},
    ]
    \node[fill={rgb,255:red,#1;green,#2;blue,#3}] () [] {};
  \end{tikzpicture}%
}

\usepackage[accepted]{icml2024}

\usepackage{amsmath}
\usepackage{amssymb}
\usepackage{mathtools}
\usepackage{amsthm}
\usepackage{xcolor}
\usepackage[capitalize,noabbrev]{cleveref}

\theoremstyle{plain}
\newtheorem{theorem}{Theorem}[section]

\newtheorem{lemma}[theorem]{Lemma}
\newtheorem{corollary}[theorem]{Corollary}
\theoremstyle{definition}
\newtheorem{definition}[theorem]{Definition}

\theoremstyle{remark}
\newtheorem{remark}[theorem]{Remark}

\usepackage[textsize=tiny]{todonotes}

\usepackage{amsmath,amsfonts,bm}

\def\eqref#1{equation~\ref{#1}}

\def\1{\bm{1}}

\def\vk{{\bm{k}}}

\def\vo{{\bm{o}}}

\def\vq{{\bm{q}}}

\def\vs{{\bm{s}}}

\def\vv{{\bm{v}}}

\def\vx{{\bm{x}}}
\def\vy{{\bm{y}}}
\def\vz{{\bm{z}}}

\DeclareMathAlphabet{\mathsfit}{\encodingdefault}{\sfdefault}{m}{sl}
\SetMathAlphabet{\mathsfit}{bold}{\encodingdefault}{\sfdefault}{bx}{n}

\def\gD{{\mathcal{D}}}

\def\gS{{\mathcal{S}}}
\def\gT{{\mathcal{T}}}

\def\sD{{\mathbb{D}}}

\def\sN{{\mathbb{N}}}

\def\sR{{\mathbb{R}}}

\newcommand{\softmax}{\mathrm{softmax}}

\DeclareMathOperator*{\argmax}{arg\,max}

\newcommand{\inner}[1]{\left\langle #1\right\rangle}

\newcommand{\abs}[1]{\left \lvert #1 \right \rvert}
\newcommand{\norm}[1]{\left \lVert #1 \right \rVert}

\newcommand{\reals}{\sR}

\newcommand{\BOS}{\left\langle{\tiny\textsc{BOS}}\right\rangle}
\newcommand{\COPY}{\left\langle{\tiny\textsc{COPY}}\right\rangle}
\newcommand{\EOS}{\left\langle{\tiny\textsc{EOS}}\right\rangle}
\newcommand{\ngram}{\mathrm{n-gram}}
\newcommand{\err}{\mathrm{err}}
\newcommand{\mem}{\mathrm{mem}}

\newcommand{\key}{\mathrm{key}}
\newcommand{\query}{\mathrm{query}}
\newcommand{\val}{\mathrm{value}}

\icmltitlerunning{Repeat After Me: Transformers are Better than State Space Models at Copying}

\begin{document}

\twocolumn[
\icmltitle{Repeat After Me: \\Transformers are Better than State Space Models at Copying\\
Transformers are Better than State Space Models at Copying}

\begin{icmlauthorlist}
\icmlauthor{Samy Jelassi}{cmsa}
\icmlauthor{David Brandfonbrener}{kempner}
\icmlauthor{Sham M.\ Kakade}{kempner,seas}
\icmlauthor{Eran Malach}{kempner}
\end{icmlauthorlist}

\icmlaffiliation{cmsa}{Harvard University, Center of Mathematical Sciences and Applications}
\icmlaffiliation{kempner}{Harvard University, Kempner Institute for the Study of Natural and Artificial Intelligence}
\icmlaffiliation{seas}{Harvard University, Departments of Computer Science and Statistics}

\icmlcorrespondingauthor{Samy Jelassi}{sjelassi@fas.harvard.edu}

\icmlkeywords{Machine Learning, ICML}

\vskip 0.3in
]

\printAffiliationsAndNotice{} 

\begin{abstract}
Transformers are the dominant architecture for sequence modeling, but there is growing interest in models that use a fixed-size latent state that does not depend on the sequence length, which we refer to as ``generalized state space models'' (GSSMs). 
In this paper we show that while GSSMs are promising in terms of inference-time efficiency, they are limited compared to transformer models on tasks that require copying from the input context.
We start with a theoretical analysis of the simple task of string copying and prove that a two layer transformer can copy strings of exponential length while GSSMs are fundamentally limited by their fixed-size latent state. 
Empirically, we find that transformers outperform GSSMs in terms of efficiency and generalization on synthetic tasks that require copying the context.
Finally, we evaluate pretrained large language models and find that transformer models dramatically outperform state space models at copying and retrieving information from context.
Taken together, these results suggest a fundamental gap between transformers and GSSMs on tasks of practical interest.
\end{abstract}

\section{Introduction}

Transformers \citep{vaswani2017attention} are the workhorse of modern sequence modeling, achieving remarkable performance on a variety of tasks, but they have unavoidable inefficiencies. Specifically, they require $ \Omega(L)$ memory\footnote{In some naive implementations of transformers, it is common to allocate a $L \times L$ matrix to compute the attention. However, memory efficient implementations, such as FlashAttention \cite{dao2022flashattention}, compute the attention with $O(L)$ memory. } and compute to predict the next token of a sequence of length $ L$. 

This has spurred a boom in attempts to create architectures that can achieve similar performance as transformers, but with $ O(1)$ memory to predict each token.
This class of models includes state space models like S4 \citep{gu2021efficiently} or Mamba \citep{gu2023mamba}, as well as traditional RNN models \citep{hochreiter1997long} and models that can be trained in parallel like linear attention \citep{katharopoulos2020transformers, choromanski2020rethinking} and parallel RNNs \citep{bradbury2016quasi, peng2023rwkv, sun2023retentive}. 
In this paper, we will refer to this entire class of models that use a fixed-size memory as ``generalized state space models'' or GSSMs (see a formal definition in Section \ref{sec:theory}).

Recent work has demonstrated impressive performance of GSSMs, but it is not yet clear what these models sacrifice for their improved efficiency, if anything.
In this paper, we find that one particular capability that is sacrificed is the ability to retrieve and repeat parts of the input context.
As a result, transformers are better than GSSMs at a variety of tasks that require accessing arbitrary parts of the context. 

To understand this gap in capabilities, we begin by presenting a theoretical analysis of the copying task\footnote{Note that we study copying of the input and \emph{not} copying of training data \citep{mccoy2023much, carlini2022quantifying}}. First, we show via construction that a simple transformer model can copy strings of length that is exponential in the number of heads of the transformer. 
This construction relies on the ability of the transformer to implement a mechanism of ``storage'' and retrieval of sequences of n tokens (n-grams), where the n-grams are used to track where to copy from. 
In contrast, we show that, trivially, GSSMs cannot accurately copy strings with more bits than the size of the latent state.

\begin{figure*}[t]
\centering
\begin{subfigure}[t]{0.31\textwidth}
\centering
\vskip 0pt
    \includegraphics[height=3.8cm]{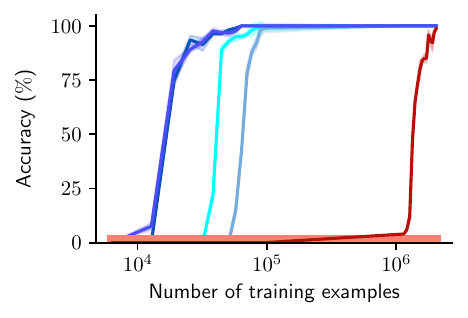}
    \small 
    \mbox{\hspace*{3cm}Transformer:} 
    \\
    \hspace*{3.7cm}GSSM: 
    \caption{Copying: training efficiency.}
    \label{fig:in_distribution_copy}
\end{subfigure}
\hfill
\begin{subfigure}[t]{0.31\textwidth}
\centering
\vskip 0pt
\includegraphics[height=3.8cm]{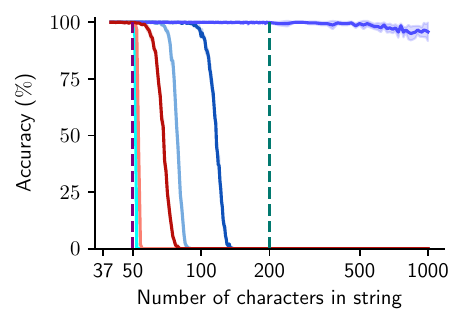}
\small 

\mbox{\hspace*{-1.8cm}\cblock{0}{255}{255} \hspace{0.5mm}RoPE\hspace{1mm} \cblock{118}{171}{223} \hspace{0.5mm}NoPE\hspace{1mm} \cblock{15}{82}{186}\hspace{1mm}Alibi\hspace{1mm} \cblock{77}{77}{255}\hspace{1mm}HAlibi}\\ 
\mbox{\hspace*{-4.1cm}\cblock{250}{128}{114}\hspace{1mm}LSTM\hspace{1.5mm}\cblock{184}{15}{10}\hspace{1mm}Mamba}
\caption{Copying: length generalization}
\label{fig:length_gen}
\end{subfigure}
\hfill
\begin{subfigure}[t]
{0.31\textwidth}
\centering
\vskip 0pt
\includegraphics[height=3.8cm]{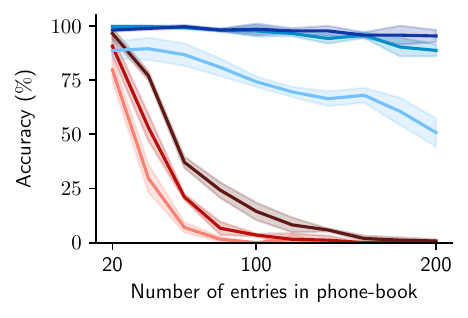}
\centering
\small
Pythia: \cblock{115}{194}{251}\hspace{1mm}410M
\hspace{1mm}\cblock{0}{142}{204}\hspace{1mm}1.4B
\hspace{1mm}\cblock{17}{52}{166}\hspace{1mm}2.8B\\
Mamba:\hspace{1mm}\cblock{250}{128}{114}\hspace{1mm}360M
\hspace*{1mm}\cblock{184}{15}{10}\hspace{1mm}1.4B\hspace{1mm}\cblock{94}{25}{20}\hspace{1mm}2.8B
\caption{Lookup with pretrained models}
\label{fig:phonebook}
\end{subfigure}
\caption{\textbf{(a) Copying: training efficiency.} Here we train models to copy strings of length $ \leq 300$ and evaluate string-level accuracy on strings of length 300. Transformers train much faster than GSSMs. An LSTM cannot even learn the task within this number of samples. 
\textbf{(b) Copying: length generalization.} Here we train models to copy on strings of length $ \leq 50$ until all models are perfect in-distribution and evaluate string-level accuracy. Purple dotted line indicates maximum training string length and green dotted line indicates context window during training. Evaluating on longer inputs, the transformer models dramatically outperform the GSSMs. Using our Hard-Alibi positional encoding, we can even generalize well beyond the training context size.
\textbf{(c) Lookup with pretrained models.} Here the task requires looking up and retrieving a number from a ``phone book'' of varying length that is entirely in context. We evaluate pretrained models 1-shot without any finetuning. Pythia (a transformer model) substantially outperforms Mamba (a GSSM) across model sizes.}
\end{figure*}

Our theory studies representation expressivity, but not whether these representations will be learned. Moreover, in practice a large GSSM may have enough capacity to represent the entire input in the latent state, at least in theory. To resolve these concerns, we conduct a variety of synthetic experiments with models of $\sim$160M parameters. We find that transformers are both much more efficient at learning to copy (Figure \ref{fig:in_distribution_copy}) and also generalize better to longer inputs (Figure \ref{fig:length_gen}). Additionally, we verify  experimentally that the copy ``algorithm'' learned by transformers indeed relies on n-grams to perform a lookup of where to copy from (\autoref{fig:copy_alg}), similarly to our theoretical construction.

Finally, we present a variety of experiments on pre-trained models to test their ability to remember and access the input context. In particular, we show that Pythia transformers \citep{biderman2023pythia} outperform Mamba GSSMs \citep{gu2023mamba} of similar size at a variety of memory-intensive tasks including copying and retrieving information from the context (Figure \ref{fig:phonebook}). This is especially notable since the Mamba models achieve lower perplexity than the Pythia models at language modeling on the Pile \citep{gao2020pile}. These experiments illustrate the practical relevance of the memory issues that we raise, and hint at one way that architectual choices can impact the downstream performance of LLMs above and beyond training perplexity.

\section{Theory: Representational Capacity}\label{sec:theory}

In this section we use the copy task for a theoretical comparison between state space models and transformers. We prove two main results. First, we construct a small transformer that solves the copy task for sequences lengths that are exponential in the transformer size. Second, we show that \textit{any} state space model \textit{fails} to solve the copy task, unless its latent state grows linearly with the sequence length. 

\subsection{Setting}

Let $\sD$ be a dictionary, which contains $D$ ``alphabet'' tokens.
A sequence-to-sequence model is a function $H : \sD^* \to \sD^*$, which maps an input sequence of tokens to an output sequence. We think of the input $x_1, \dots, x_i$ as the ``prompt'' to the model, and of the output sequence $H(x_1, \dots, x_i)$ as the generated ``answer''.

A sequence-to-token mapping is a function $h : \sD^* \to \sD$. Any sequence-to-token model $ h $ naturally defines a sequence-to-sequence model $ H$ by auto-regressive inference.
Namely, for every input sequence $x_1, \dots, x_i \in \sD$ we define recursively $x_{i+j} = h(x_1, \dots, x_{i+j-1})$ and let $H(x_1, \dots, x_i) = (x_{i+1}, x_{i+2}, \dots)$.

\paragraph{Generalized state space models.}
A state space $\gS$ is some finite set. We denote by $\mem(\gS)$ the number of bits required to encode the states of $\gS$, namely $\mem(\gS) = \log(\abs{\gS})$. A \emph{generalized state space model} (GSSM) is a sequence model defined by an update rule $u : \gS \times \sD  \to \gS$ and some output function $r : \gS \to \sD$. Let $s_0 \in \gS$ be some initial state. Given some sequence $x_1, \dots, x_L$, the state of the model at iteration $i$ is denoted by $S_i(x_1, \dots, x_i)$ and the output token is denoted by $R_i(x_1, \dots, x_i)$. The state and output are defined recursively:
1) $S_0(\emptyset) = s_0$, 2) $S_i(x_1, \dots, x_i) = u(S_{i-1}(x_1, \dots, x_{i-1}), x_i)$, 3) $R_i(x_1, \dots, x_i) = r(S_{i}(x_1, \dots, x_{i}))$.

\begin{figure}[t]
    \hspace*{-.2cm}\includegraphics[width=\linewidth]{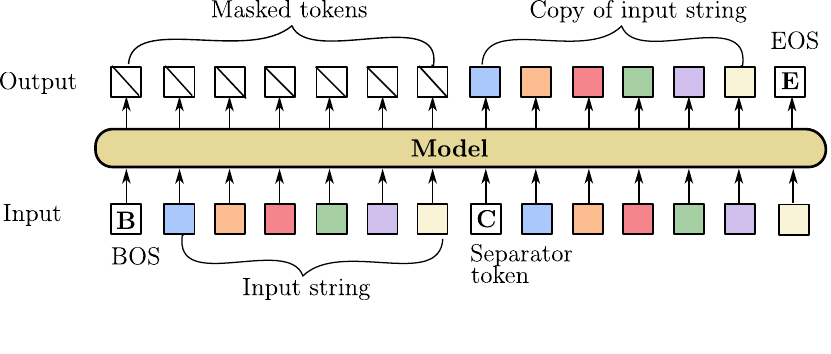}
    \vspace{-0.7cm}
    \caption{An illustration of the copy task.}
    \label{fig:copy_task}
    \vspace{-0.5cm}
\end{figure}

\begin{remark}
    It is important to note that for any sequence model, there are two types of memory considerations: 1) input-independent memory (\emph{parameters}) and 2) input-dependent memory (\emph{activations}). The GSSM definition constraints the input-dependent memory (\emph{activations}), which corresponds to $\mem(\gS)$, and does not restrict in any way the amount of input-independent memory (\emph{parameters}) or the run-time of state updates. Since our main goal is to show a lower bound on the state space memory, leaving all other considerations unconstrained only strengthens our results.
\end{remark}

\paragraph{Transformers.} 
Given some input of length $L$ and dimension $d$, denoted $\vx_1, \dots, \vx_L \in \reals^d$, an attention head is parameterized by $W_k, W_q, W_v \in \reals^{d \times d}$. We denote $\vk_i = W_k \vx_i, \vq_i = W_q \vx_i, \vv_i = W_v \vx_i$ and denote $K_i = [\vk_1, \dots, \vk_i] \in \reals^{d \times i}$ and $V_i = [\vv_1, \dots, \vv_i] \in \reals^{d \times i}$. We denote the output of the head at token $i$ by $\vo_i \in \reals^d$, where $\vo_i = V_i \cdot \softmax(K_i \cdot \vq_i)$.

We consider a transformer with $l$ attention heads, each one of dimension $d$ so that the full dimension of the Transformer is $dl$. 
An embedding is some mapping $\Psi : \sD \to \reals^d$. An MLP is a function $f : \reals^{dl} \to \reals^{dl}$ s.t. $f(\vx) = U_1 \sigma (U_2 \vx)$, for some activation function $\sigma$. Both the embedding and the MLP layer are assumed to be applied on the token level. An attention-block is a set of $l$ heads applied in parallel, and a transformer-block is an attention-block followed by an MLP which operates on the concatenated output of the $l$ heads. The output of the model is sampled based on the output of the final layer. For simplicity, we study the $\argmax$ ``sampling'' (i.e., predicting the most probable token).

\paragraph{The copy task.}
To define the copy task, we add two special tokens to $ \sD$: (1) \emph{beginning-of-sequence} token, denoted $\BOS$, and (2) \emph{copy} token, denoted $\COPY$. So now $ |\sD| = D + 2$.
A length-$L$ copy distribution $\gD_L$ over $\sD^{L+2}$ generates strings of the form: ``$\BOS, x_1, \dots, x_L, \COPY$'', where $\vx \in (\sD \setminus \{\BOS, \COPY\})^L$.

For some sequence-to-sequence model $H : \sD^* \to \sD^*$, we denote the error of $H$ on a copy distribution $\gD_L$ by
\[
    \err_{\gD_L}(H) = \Pr_{\gD_L}\left[H_{1:L}(\BOS, \vx, \COPY) \ne \vx\right]
\]
where $H_{1:L}(\cdot)$ denotes the first $L$ tokens generated by $H$. That is, we expect the model to output an exact copy of $\vx$.

\subsection{Transformers can copy inputs of exponential length}

In this section, we show that transformers can implement the copy operation for input sequences with length exponential in the number of heads. Namely, we construct a transformer with two blocks that gets small error on the copy task. 

\paragraph{Construction: hash-based copying.} The key idea in the construction is to first ``hash'' sequences of $n$ tokens ($n$-grams), then at each iteration of the auto-regression attend to the previous occurrence of the most recent $n$-gram, and output the succeeding token.
That is, we show that a transformer can implement the copying algorithm illustrated in \autoref{fig:copy_alg} (and see also Algorithm \ref{alg:copy} in the Appendix).

\paragraph{Positional embedding: Hard-ALiBi.}
To perform the hashing described in the algorithm, we need to be able to leverage local positional information to define a hash, and also to apply this hash function globally on the entire input.
To do this, we use a hard version of ALiBi \citep{press2021train}, which we call \emph{Hard-ALiBi}. Just as in ALiBi, we add a bias $ b_i $ to the $i$-th attention head as follows:
$\label{eq:hard-alibi} \vo_i = V_i \cdot \softmax(K_i \cdot \vq_i + b_i)$.
Specifically, we set $b_i$ s.t. $b_{i,j} = -\infty$ for $j \le i-m$ and $b_{i,j} = 0$  for $j > i-m$. We allow different heads with different choices of $m$ and also allow for $m=\infty$ which corresponds to softmax attention with no positional embedding. This is illustrated in Figure \ref{fig:hard_alibi} (Appendix). While the Hard-ALiBi is introduced  for our theoretical construction, we observe it also offers significant benefits empirically, as discussed in Section \ref{sec:learning}.

\begin{figure}[t]
    \centering \includegraphics[width=\linewidth]{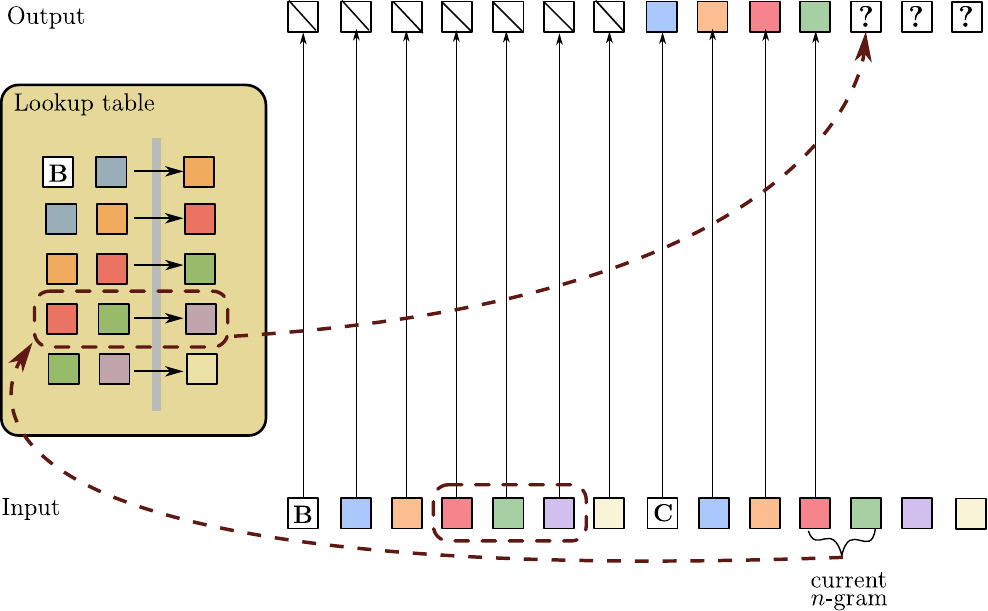}
    \vspace{-0.4cm}
    \centering
    \caption{An illustration of the $n$-gram based copy algorithm. In order to predict the next token, we match the current $n$-gram to the corresponding $n$-gram in the input, then output the next token.}
    \label{fig:copy_alg}
    \vspace{-0.2cm}
\end{figure}

\paragraph{Guarantees.}
The copy algorithm given in Algorithm \ref{alg:copy} (and similarly, our transformer construction) can perfectly copy the input sequence, as long as there are no repeated $n$-gram patterns in the input. Therefore, the error of the algorithm depends on the probability of repeated $n$-grams:

\begin{definition}
    Let $\gD_L$ be some copy distribution. For some $n \in \sN$, let $p_{\ngram}(\gD_L)$ be the probability that $x_1, \dots, x_L$ contains two repeated sequences of $n$ tokens. Namely:
    \begin{align*}
        p_{\ngram}(\gD_L) = \Pr_{\gD_L} \left[\exists_{i \ne j}~\mathrm{s.t.}~x_i, \dots x_{i+n} = x_{j}, \dots, x_{j+n}\right]
    \end{align*}
\end{definition}

Below we state the main theoretical result on copying with transformers, showing that transformers can copy their input, with error bounded by the probability of repeated $n$-grams:

\begin{theorem}
    \label{thm:upper_bound}
    For all $n$, there exists a depth-2 transformer $\gT$ of dimension $O(n\log(D))$ s.t. for all $2n \le L \le D^n$, and for any copy distribution $\gD_L$, $\err_{\gD_L}(\gT)< p_{\ngram}(\gD_L)$.
\end{theorem}

Intuitively, the probability of repeated $n$-grams decays quickly when increasing the value of $n$. Indeed, we show that for the uniform distribution over sequences, this probability decays \emph{exponentially} with $n$:

\begin{lemma}
    \label{lem:qgram}
    Let $\gD_L$ be the copy distribution generated by sampling $\vx$ from the uniform distribution over the ``alphabet'' (non-special) tokens. Then, $p_{\ngram}(\gD_L) < L^2 D^{-n}$.
\end{lemma}

Combining the above results, we get that transformers can copy sequences of tokens drawn from the uniform distribution, using a number of parameters that depends only \emph{logarithmically} on the input sequence length.

\begin{corollary}
    Fix some $\epsilon \in (0,1/2)$ and some $L \ge \Omega(\log(1/\epsilon))$. There exists a depth-2 transformer $\gT$ of dimension $O(\log(L/\epsilon)\log(D))$ s.t. for the uniform copy distribution $\gD_L$, $\err_{\gD_L}(\gT) < \epsilon$.
\end{corollary}

\begin{remark}
    For simplicity we do not limit the precision of the parameters or activations, but note that our results hold for finite-precision transormers, using $O(\log(\log(L)))$ bits.
\end{remark}

\subsection{State Space Models cannot copy inputs beyond memory size}

We saw that transformers are able to copy uniform sequences of tokens, with parameter count \emph{logarithmic} in the sequence length. We now show that GSSMs \emph{cannot} copy uniform input sequences, unless the capacity of their state space grows \emph{linearly} with the size of the sequence length. This is intuitive: to be able to copy the entire input sequence, the model needs to store it in its state space, which requires the memory to grow linearly with the sequence length.

\begin{theorem}
    \label{thm:lower_bound}
    Fix some GSSM $H$ over state space $\gS$. Then, for all $L$, for the uniform copy distribution $\gD_L$, the model $H$ has error $\err_{\gD_L}(H) > 1-\frac{\abs{\gS}}{D^L}$.
\end{theorem}

Given Theorem \ref{thm:lower_bound}, the following Corollary is immediate:

\begin{corollary}
    \label{cor:lower_bound}
    Fix some $L \in \sN$. Then, every GSSM $H$ with state space $\gS$ s.t. $\mem(\gS) < L\log(D)-1$ has error $\err_{\gD_L}(H) > 1/2$ for the uniform copy distribution $\gD_L$.
\end{corollary}

\begin{remark}
As mentioned previously, the input-dependent memory of transformers grows linearly with the sequence length, which is \emph{less} memory-efficient compared to GSSMs. However, it is interesting to note that from the above result, at least for the copy task, transformers are \emph{almost optimal} in terms of their input-dependent memory. More specifically, an implication of Theorem \ref{thm:upper_bound} is that there exists a transformer which can copy inputs of length $L$ using $\tilde{O}(L)$ input-dependent memory\footnote{We use $\tilde{O}$ to hide logarithmic factors.}, and due to Corollary \ref{cor:lower_bound} this is indeed optimal (up to logarithmic factors).
\end{remark}

\section{Learning to Copy}

\label{sec:learning}

 In the previous section, we proved that transformers can represent the copy operation for exponentially long sequences, while GSSMs fail to copy long sequences due to their limited memory. While these results show that in theory, transformers can outperform GSSMs, our theoretical results do not establish that such a gap will be observed in practice for two reasons. First, it is not clear that transformers can indeed \emph{learn} to copy from examples. Second, GSSMs in practice may use a large latent state memory, so that our bounds only hold for very long sequences of tokens. 
For example, a latent state of 1000 32-bit floating point numbers has enough bits to store at least 2000 tokens from a 50K token vocabulary. However, even though a GSSM could fit the context into memory, it may not learn to do so.

Our goal in this section is to verify that our theoretical analysis bears out experimentally when training models from scratch on synthetic data, before moving on to study pretrained models in the next section.
Specifically, we train transformers and GSSMs (LSTM \citep{hochreiter1997long} and Mamba \citep{gu2023mamba}) on variants of the copy task shown in \autoref{fig:copy_task}.

\subsection{Experimental setup}

We now provide a brief overview of our experimental setup. Further details may be found in \autoref{app:experiments}. Code and data available at:  \url{https://github.com/sjelassi/transformers_ssm_copy}

\paragraph{Architecture.} In all our experiments, we set the model hyperparameters so that the Mamba and transformers have a similar number of parameters ($\approx 160$ million parameters). Since we find that large LSTMs are hard to train (as confirmed in \citet{pascanu2013difficulty}), we use the largest LSTM we managed to train which has $\approx 40$ million parameters.

\paragraph{Dataset.} During training, we generate in an online manner a batch of 64 examples at each epoch. At test time, we evaluate our models on $10$ batches of $128$ examples. We report the mean and standard-deviation over these 10 batches. If not specified otherwise, our token space $\mathcal{V}$ is of size 30 and made of the alphabet letters i.e.\ $\mathcal{V}=\{a,\dots,z,\BOS,\EOS,\COPY\}$ where $\BOS$ is the beginning of sentence token, $\EOS$ the end of sentence token and $\COPY$ the separator token. All the strings are sampled uniformly i.e.\ we first sample the length of the sequence and then independently sample each position of the string from $\mathcal{V}$. Finally, we “pack the context” with i.i.d. sequences during training similarly to \cite{zhou2023algorithms}: we fill the context with multiple
independent samples of the task. 

\paragraph{Positional information.} 
Positional information also plays an important role in the length generalization capacity of Transformers \citep{jelassi2023length,kazemnejad2023impact,shen2023positional}. Previously popular methods of input-layer positional embeddings (e.g. sinusoidal \cite{vaswani2017attention} or learned \citep{Radford2019LanguageMA}) have been replaced by relative positional encodings at each attention layer (e.g. RoPE \citep{su2023roformer}, Alibi \citep{press2021train}, or NoPE \citep{kazemnejad2023impact}). Below, we experiment these positional encodings along with the Hard-Alibi encoding introduced in \autoref{sec:theory}.

\begin{figure}
    \centering
    \includegraphics[height=4.2cm]{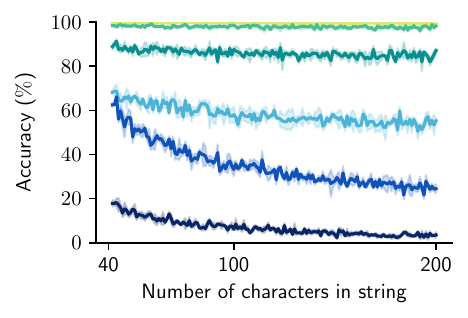}\\
    \centering
    \small
    n-gram length: {\small \cblock{255}{213}{0} \hspace{1mm}2\hspace{1.5mm}\cblock{220}{231}{117} \hspace{1mm}3\hspace{1.5mm} \cblock{72}{191}{145}\hspace{1mm}4\hspace{1.5mm}\cblock{8}{143}{143}\hspace{1mm}5\hspace{1.5mm}\cblock{77}{180}{215}\hspace{1mm}6\hspace{1.5mm}
    \cblock{15}{82}{186}\hspace{2mm}7\hspace{1.5mm}\cblock{8}{37}{103}\hspace{1mm}8}
    \vspace{-0.2cm}
    \caption{String-level copying accuracy on data with duplicated n-grams. Copying fails when the duplicated n-gram is too long as the model can no longer perform n-gram lookups.}
    \label{fig:n_gram_dup}
    \vspace{-0.3cm}
\end{figure}

\subsection{Data efficiency on the copy task}

We begin by training our models on the simple task of copying a sequence of input tokens described in \autoref{fig:copy_task}. The model gets an input of $ \leq L$ tokens followed by a \emph{Separator} ($\COPY$) token, and needs to output the same sequence again from the beginning. In this section, we focus on in-distribution learning: we train on strings of random length $ \leq L = 300$ and record the string-level accuracy on evaluation strings sampled from the training distribution.
Results for this experiment are shown in \autoref{fig:in_distribution_copy}. Clearly, there is a large gap between the transformers and GSSMs. We observe that the transformers need 100x less samples than the best GSSMs to learn the copy task. 

Note that the sharp changes in accuracy displayed in \autoref{fig:in_distribution_copy} are due to the log-scaled x-axis and choice of string-level accuracy as a metric. In \autoref{fig:char_level_id},  we report the character-level accuracy, which yields smoother curves demonstrating the learning process of GSSMs. Regarding LSTMs, we find that they do not manage to learn on length-300 strings even at the character level. In \autoref{fig:str_level_short}, we show that LSTMs are able to learn to copy on shorter strings and that string length is the bottleneck.

\subsection{Length generalization on the copy task}

The prior experiment demonstrates superior efficiency of learning in-distribution. Now, we test the ability of the learned functions to generalize out-of-distribution. 
Specifically, we consider generalization from short sequences to longer sequences. Testing this sort of generalization can help us to better understand which function the model has learned, i.e. whether the model has truly learned the ``correct'' copy operation or whether it just learned to copy sequences of the particular size it was trained on.

Here, we train all models on sequences of $\le 50$ tokens, and test them on sequences of up to $1000$ tokens, reporting string-level accuracy.
As seen in \autoref{fig:length_gen}, all models are able to (eventually) solve the task in-distribution on lengths of $ \leq 50$, but transformer-based models display much better generalization to longer inputs compared to GSSMs. 
Namely, we observe that the performance of the GSSMs (LSTM and MAMBA) drops to zero almost immediately when increasing the input length, while the performance of transformers decays much more gradually with length. 

\paragraph{Positional information.} When looking at the relative performance of different transformer models in \autoref{fig:length_gen}, it becomes clear that the positional encoding is important to length generalization.  Specifically, the ALiBi and NoPE transformers dramatically outperform the RoPE model on longer inputs. This is likely because the sinusoidal embeddings of RoPE create a more dramatic change than the decay of ALiBi or NoPE when we go to longer inputs.

\paragraph{Improved generalization with Hard-ALiBi.} 
To test our understanding of how transformers learn to copy, we now consider swapping in the Hard-ALiBi positional encoding that we used in our theoretical construction of hash-based copying (introduces in \autoref{eq:hard-alibi} and illustrated in \autoref{fig:hard-alibi-drawing} in the Appendix).  
\autoref{fig:length_gen} shows that a transformer trained with Hard-ALiBi embedding on sequences of length $\le 50$ achieves almost perfect length generalization up to sequences of length 1000. Note that this is well beyond the context length ever encountered in training.

\subsection{Transformers learn to use n-gram hashing}
Next, we attempt to determine whether the transformer trained on the copy task indeed applies the mechanism of storage and retrieval of n-grams. 
To do this, we evaluate the performance of a transformer with Hard-ALiBi positional encoding 
trained on the copy task when tested on a distribution of examples that intentionally contains duplicate n-grams. 
That is, we draw uniform sequences of tokens, and then randomly replace some n-gram with another n-gram that already appears in the sequence, such that each example always contains two copies of the same n-gram (typically followed by a different token). We use the Hard-Alibi model here since it performs the best for the copy task as showed in \autoref{fig:in_distribution_copy}.
\autoref{fig:n_gram_dup} shows the performance of the transformer for different choices of $n$. We observe that the transformer maintains roughly the same accuracy for $n\le 4$, but that its accuracy starts dropping when the inputs contains duplicate sequences of 5 or more tokens. 
This suggests that the transformer relies on something like 5-gram retrieval to do the copy task. \autoref{fig:n_gram_perf} further strengthens this point. We report the performance of perfect n-gram models in the copy task and observe that the performance of Transformers enhanced with Hard-ALiBi matches with the one of a 5-gram model.

\subsection{GSSMs cannot arbitrarily retrieve from context}

\begin{figure}[t]

\begin{subfigure}[t]{0.48\textwidth}

\centering
\includegraphics[width=0.9\linewidth]{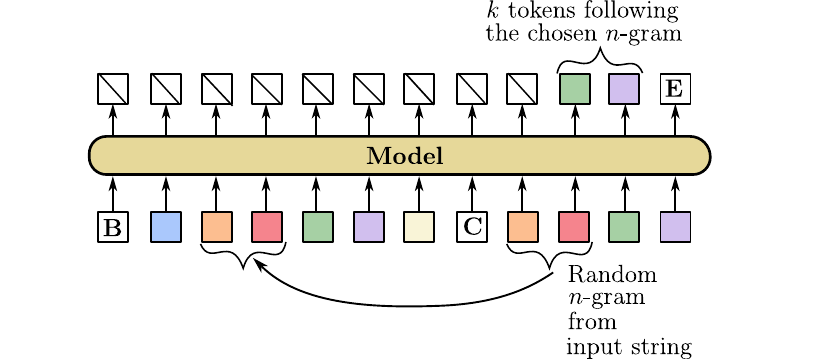}
\end{subfigure}
\begin{subfigure}[t]{0.48\textwidth}

\centering

\includegraphics[height=4.2cm]{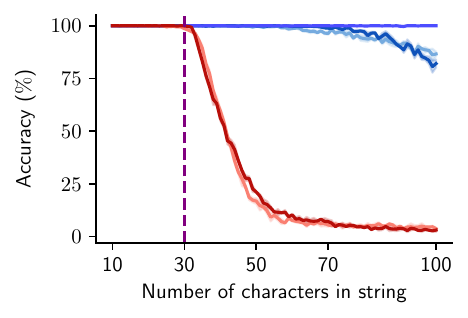}
\end{subfigure}
\small 
\centering
Transformer: \cblock{118}{171}{223} \hspace{0.5mm}NoPE\hspace{1mm} \cblock{15}{82}{186}\hspace{1mm}Alibi\hspace{1mm} \cblock{77}{77}{255}\hspace{1mm}HAlibi\\ GSSM: \cblock{250}{128}{114}\hspace{1mm}LSTM\hspace{1.5mm}\cblock{184}{15}{10}\hspace{1mm}Mamba
\vspace{-0.2cm}
\caption{\small\textbf{Top}: An illustration of the suffix key veriant of the n-gram lookup task.  \textbf{Bottom:} When trained on strings of length $ \leq 30$, transformers outperform GSSMs on longer inputs, illustrating superior performance on this memory-intensive task.}
\vspace{-0.2cm}
\label{fig:n_gram_completion}
\end{figure}

\begin{figure}[t]
\begin{subfigure}[t]{0.48\textwidth}

\centering
\includegraphics[width=0.9\linewidth]{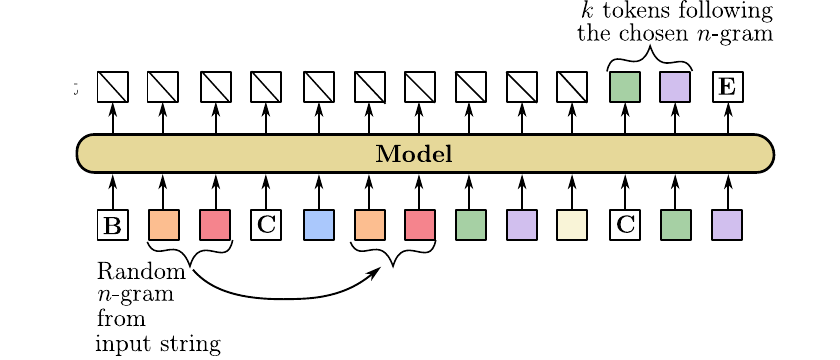}
\end{subfigure}
\begin{subfigure}[t]{0.48\textwidth}

\centering
\includegraphics[height=4.2cm]{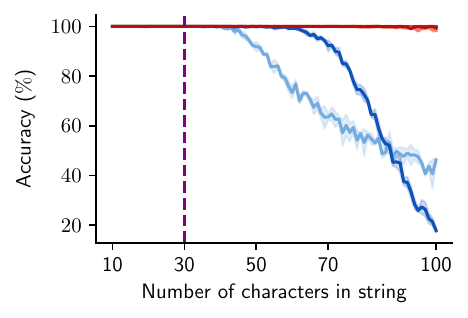}

\end{subfigure}

\small 
\centering
Transformer: \cblock{118}{171}{223} \hspace{0.5mm}NoPE\hspace{1mm} \cblock{15}{82}{186}\hspace{1mm}Alibi\hspace{1mm} \cblock{77}{77}{255}\hspace{1mm}HAlibi\\ GSSM: \cblock{250}{128}{114}\hspace{1mm}LSTM\hspace{1.5mm}\cblock{184}{15}{10}\hspace{1mm}Mamba
\vspace{-0.2cm}
\caption{\small\textbf{Top}: An illustration of the prefix key veriant of the n-gram lookup task.  \textbf{Bottom:} When trained on strings of length $ \leq 30$, GSSMs perform as well as the Hard-Alibi transformer and better than the other transformers. This slight variant of the task requires much less memory and is thus more suitable to the strengths of GSSMs at storing a small state over time.}
\label{fig:n_gram_completion_prefix}
\vspace{-0.2cm}
\end{figure}

We now introduce another task to probe the mechanisms that the models use to copy from the context: the \emph{n-gram lookup} task. In this task the model needs to use a given n-gram as a key to look up the k-token key that follows the query. We consider two variants of the task: \emph{suffix keys} and \emph{prefix keys}. In both variants, we assess length generalization to understand the function that the models have learned.

First, we consider the suffix key version of n-gram lookup.
In this task, the model is given a sequence $L$ of input tokens, a separator, and then an n-gram from the input sequence. 
The model then needs to output a sequence of $k$ tokens following the chosen n-gram (see \autoref{fig:n_gram_completion} for an illustration). This task is closely related to induction heads \citep{olsson2022context}. This task requires the model to be able to ``store'' the entire context in order to effectively find the correct key to access it's query.
We train all models on sequences of at most 30 tokens and show results in \autoref{fig:n_gram_completion}. Transformers perform well on this task, with a relatively small drop in performance when increasing the sequence length up to 100. This suggests that transformers can learn to perform n-gram storage and retrieval. 
GSSMs, however, perform poorly beyond their training distribution. 
Intuitively, this task still requires the models to store the entire input sequence, something that GSSMs struggle to do. 

Next, we try the prefix key version of n-gram lookup. Here we provide the n-gram key at the beginning and then the full input sequence (illustrated in \autoref{fig:n_gram_completion_prefix}).
In this version of the task the model does not need to store the entire input since it can look for the key on the fly as the sequence is processed.
This is good for the GSSMs, since they can write the key into the state and then ignore inputs that do not match.
Indeed, GSSMs achieve perfect length-generalization on this variant.
Interestingly, the GSSMs even outperform the NoPE and ALiBi transformers (although not the Hard-Alibi model). 
We hypothesize that this may be an issue where these positional embeddings make it more difficult to effectively perform the hashing lookup over a long distance in relative positions.
Taken together, these results illustrate how GSSMs seem to be memory limited, but can be effective when the tasks only require a summary of the inputs rather than storing the entire context.

\section{Pre-trained Models}\label{sec:pretrained}

In this section, we compare the performance of pre-trained transformers and pre-trained GSSMs on memory-intensive tasks such as copying long strings, retrieval and few-shot question answering. We show that transformers outperform GSSMs of similar scale on such memory-intensive tasks, even when the GSSM has lower perplexity as a language model. These results confirm that the limitation of GSSMs raised in previous sections apply to large scale models trained on real pretraining data.

\subsection{Setup}

In the experiments below, we compare Pythia transformer models \cite{biderman2023pythia} of sizes ranging from 410M to 2.8B against Mamba models \cite{gu2023mamba} of similar sizes. All these models have been pre-trained on the Pile \cite{gao2020pile} and use the same tokenizer. The Mamba models generally have slightly lower perplexity on the training set for a given size. The main difference between the Pythia and the Mamba models is their architectural design.

We compare these models by measuring their performance while varying the input instance length and consider two types of tasks: copy-based and information retrieval tasks. The copy-based tasks consist of presenting a random text to the model and asking it to copy the text. In the information retrieval tasks, we provide a text to the model and ask it a related question. These retrieval tasks can be seen as ``selective copy'', since the model needs to copy a small chunk of the input text in order to respond to the question. To measure performance, we use the string-level accuracy in all the experiments except in \autoref{fig:squad} where we consider question answering and thus report the F1 score.
We evaluate the models over 10 batches of size 64 for all the tasks except for question answering where we evaluate over 50 questions because the number of questions with a given context length is limited. Further details are in \autoref{app:experiments}.

\begin{figure*}[t]
\centering
\begin{subfigure}[t]{0.31\textwidth}
\centering
\vskip 0pt
    \includegraphics[height=3.8cm]{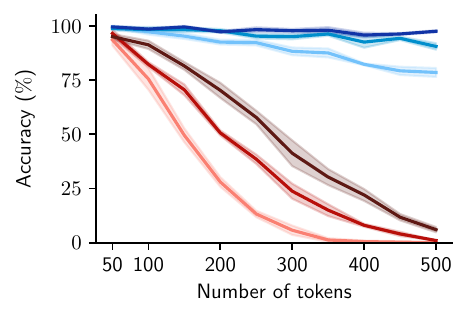}
    \small
    Pythia: \cblock{115}{194}{251}\hspace{1mm}410M
\hspace{1mm}\cblock{0}{142}{204}\hspace{1mm}1.4B
\hspace{1mm}\cblock{17}{52}{166}\hspace{1mm}2.8B \\
Mamba:\hspace{1mm}\cblock{250}{128}{114}\hspace{1mm}360M
\hspace*{1mm}\cblock{184}{15}{10}\hspace{1mm}1.4B\hspace{1mm}\cblock{94}{25}{20}\hspace{1mm}2.8B
    \caption{Copy: natural language strings}
    \label{fig:copy_c4}
\end{subfigure}
\hfill
\begin{subfigure}[t]{0.31\textwidth}
\centering
\vskip 0pt
\includegraphics[height=3.8cm]{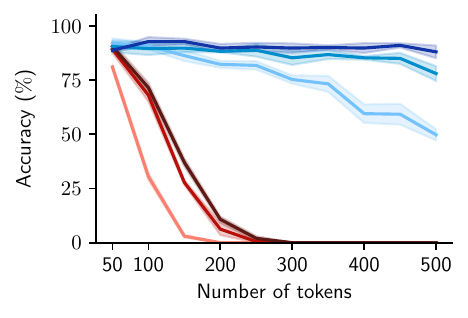}
\centering
\small 
Pythia: \cblock{115}{194}{251}\hspace{1mm}410M
\hspace{1mm}\cblock{0}{142}{204}\hspace{1mm}1.4B
\hspace{1mm}\cblock{17}{52}{166}\hspace{1mm}2.8B \\
Mamba:\hspace{1mm}\cblock{250}{128}{114}\hspace{1mm}360M
\hspace*{1mm}\cblock{184}{15}{10}\hspace{1mm}1.4B\hspace{1mm}\cblock{94}{25}{20}\hspace{1mm}2.8B
\caption{Copy: shuffled strings}
\label{fig:copy_shuffle}
\end{subfigure}
\hfill
\begin{subfigure}[t]
{0.31\textwidth}
\centering
\vskip 0pt
\includegraphics[height=3.8cm]{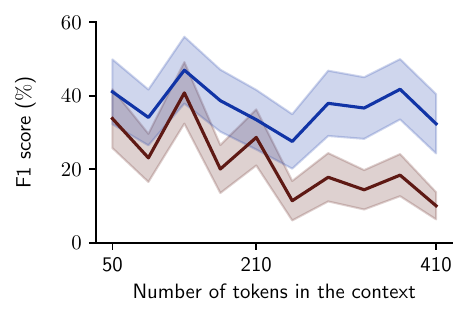}
\small
Pythia: \cblock{17}{52}{166}\hspace{1mm}2.8B\\
Mamba: \cblock{94}{25}{20}\hspace{1mm}2.8B
\caption{Question answering (SQUAD)}
\label{fig:squad}
\end{subfigure}

\centering

\small
\caption{\textbf{(a) Copy: natural language strings.} We compare pretrained models on their ability to copy natural language strings sampled from C4 of varying lengths and report string-level accuracy. The transformer models substantially outperform the GSSMs.  
\textbf{(b) Copy: shuffled strings.} To test whether it mattered that the strings were in natural language, we randomly shuffle the word order of the strings from the previous experiment. We find that this degrades performance, especially for the Mamba models. 
\textbf{(c) Question answering (SQUAD).} We compare Pythia and Mamba on a standard question answering dataset where we bin the dataset based on the length of the context paragraph. We find that Mamba performance decays more quickly with the length of the context.}
\end{figure*}

\subsection{Copying the input context}\label{subsec:copy_pretrained}

We first observe that pre-trained transformers outperform pre-trained GSSMs at copying long natural language strings. In \autoref{fig:copy_c4}, we randomly sample strings from  the C4 dataset \cite{raffel2020exploring} with varying number of tokens. Our prompt consists of two copies of the sampled string plus the first word of the string and we expect the model to complete the third copy. Even the smallest transformer model dramatically outperforms the largest GSSM. This happens even though the large GSSMs have enough bits in the state variable to potentially store the context. This confirms the idea that this is an architectual bias of transformers that makes it easier for them to copy from the context.

Unlike strings of tokens sampled uniformly at random, natural text can often be compressed, possibly allowing language models to copy longer strings even with limited memory. To test whether this matters, in \autoref{fig:copy_shuffle} we conduct the same experiment as above but randomly shuffle the order of the words in the strings. We find that when we shuffle the words, both GSSMs and transformers perform worse on the task, but the effect is more stark for GSSMs. Even the largest GSSM now gets zero accuracy on strings of length 300. This suggests that when the input is more difficult to compress, the GSSM suffers due to its fixed size state.

\subsection{Retrieval from the input context}

While copying provides a clear task to separate the model classes, it is not a particularly realistic task. That said, it presents an extreme case of a type of behavior that is highly relevant for many tasks of interest. In particular, many tasks require retrieving specific information from the context that is relevant to the desired output. 
This subsection presents examples of how our results transfer to more practical tasks.

\paragraph{Phone-book lookup.} We first consider a ``phone-book'' experiment where we provide a synthetic phone-book to the model and ask it to return the phone number when given a name. We generate the phone-book by randomly sampling $L$ names and their associated phone number. One line of this phone-book looks like ``John Powell: 609-323-7777". Our prompt to the model consists of the phone-book, two few-shot examples and a question asking for the phone number of a randomly sampled name from the phone-book. \autoref{fig:phonebook}
 reports the accuracy obtained by the pretrained transformers and GSSMs while varying the size of the phone-book $L.$ We observe that even the smallest transformer (410M parameters) outperforms the largest GSSMs (2.8B parameters) when the phone-book size is long enough ($L\geq 70$).
This shows that in retrieval tasks which require access to the whole context, GSSMs struggle to store the relevant information in their fixed-size state.

\paragraph{Question-Answering.} In this experiment, we compare the 2.8B parameter Mamba and transformer models\footnote{In our experiments, smaller models were unable to achieve reasonable and consistent performance on this dataset.}, on the SQuAD question-answering dataset \cite{rajpurkar2018know}. This dataset provides text paragraphs together with a few questions regarding the text. We probe the models to answer the question by providing a single demonstration of a question/answer pair (corresponding to the same text) before giving the target question. We bin the paragraphs according to their lengths, and report the F1 score as a function of the paragraph length for both models in \autoref{fig:squad}.
We observe that while for short paragraphs, both the Pythia transformer and Mamba achieve comparable performance, the performance of Mamba degrades more quickly with the paragraph length, while the transformer-based model maintains a similar accuracy even for longer texts.
This result shows that the fixed-memory of GSSMs also limits their performance on standard natural tasks.

\section{Related Work}

There exists a broad body of prior work on the representational capacity of GSSMs like RNNs 
\citep{merrill2019sequential, merrill2020formal} as well as transformers \citep{weiss2021thinking, merril2022tacl, wei2022statistically, sanford2023representational, edelman2022inductive}. 
Previous works that study transformers do so through comparison to other complexity classes, such as threshold circuits \cite{merril2022tacl}, RASP language \cite{weiss2021thinking} or first-order logic \cite{chiang2023tighter} (see \citet{strobl2023transformers} for a thorough review). 
These works do not provide insights into how transformers implement algorithms for solving specific problems. 
In contrast, our theoretical result constructs a transformer for the copy task, which illustrates the mechanism and provides tight bounds on the model size. 
Together with the result showing that GSSMs cannot copy long sequences, our theory characterizes the power of different sequence models on the copy task. 
Other theoretical separation results between transformers and RNNs \cite{sanford2023representational, merrill2019sequential} use more complex tasks of less practical relevance.

Other papers have previously demonstrated the capacity of transformers to leverage the entire input context for tasks like retrieval, question answering, and in-context learning \citep{devlin2018bert, raffel2020exploring, petroni2020context, brown2020language, liu2023lost,needlehaystack}. Another line of work has studied the ``induction head'' mechanism in transformers that performs a retrieval operation much like the one we observe for copying \citep{olsson2022context}.
But, to our knowledge, there is not a comparison in related work between transformers and GSSMs of similar quality on these tasks.

Several of our experiments study length generalization as a way to assess whether the model found the ``right way'' to solve the task. 
Prior work on length generalization in transformers has focused on the data distribution \citep{anil2022exploring}, positional embeddings \citep{kazemnejad2023impact}, and arithmetic tasks \citep{deletang2022neural,ruoss2023randomized,jelassi2023length,zhou2023algorithms}. We extend many of these ideas to the copying task.

Finally, we note that while we focus on tasks where transformers outperform GSSMs, there are also tasks where GSSMs outperform transformers. For example, 
\citet{liu2023exposing} shows that transformers fail to generalize out of distribution for ``flip-flop language modeling'', while LSTMs do so easily. These tasks require tracking a small $ O(1)$ state variable over time. Another benefit of GSSMs is the ability to input long contexts like DNA sequences that may be impractical for transformers \citep{nguyen2023hyenadna}. 

Concurrently to our work, \citet{akyurek2024context, grazzi2024mamba, park2024can} studied the difference between Transformers and Mamba at in-context learning, which can be seen as a form of copying. In particular, \citet{akyurek2024context} finds that Transformers have an advantage over other architectures at this task because they have ``n-gram heads''. Similarly to these works, we hint the limitations of SSMs in memory-intensive tasks such as copying because of their limited state size. We also show that Transformers can perform copying using the Hard-ALiBi positional encoding, which improves the model's ability to learn $n$-gram matching.

\section{Discussion}

We have demonstrated through theory and experiments that transformers are better than GSSMs at copying from their input context. However, we emphasize that state space models have many advantages over transformers. The memory and computational complexity of GSSMs does not increase with the input length, which is ideal for training and inference on long inputs. Additionally, state space models such as RNNs are better at tracking state variables across long sequences \cite{liu2023exposing}, which may be useful for generating long consistent text. Importantly, language processing in the human brain appears to be much more similar to how state space models process language \cite{tikochinski2024incremental}.

We therefore believe that future work should focus on building hybrid architectures that endow state space models with an attention-like mechanism, allowing them to retrieve relevant pieces of text from their input. Indeed, humans have an incredibly limited capacity for memorizing sequences \cite{miller1956magic}, but can translate entire novels if we allow them to look back at the text \cite{SheltonQuixote}.

\subsection*{Impact Statement}

This paper presents work whose goal is to advance the field of Machine Learning. There are many potential societal consequences of our work, none which we feel must be specifically highlighted here.

\section*{Acknowledgements}

We thank Boaz Barak for helpful discussions. Kempner Institute computing
resources enabled this work. Samy Jelassi acknowledges funding supported by the Center of Mathematical Sciences and Applications. This work has been made possible in part by a gift from the Chan Zuckerberg Initiative Foundation to establish the Kempner Institute for the Study of Natural and Artificial Intelligence. Sham Kakade acknowledges funding from the Office of Naval Research under award N00014-22-1-2377.

\bibliography{icml2024/references}
\bibliographystyle{icml2024}
\vfill
\appendix
\onecolumn

\section{Experimental setup }\label{app:experiments}

In this section, we provide additional details about our experimental setup. We first give a description of the positional encodings used in our transformers experiments (\autoref{app:transformers}) and then give details about the training and evaluation procedures (\autoref{app:train_eval_details}).

\subsection{Positional encodings in the transformers}\label{app:transformers}

\begin{figure}[h!]
\centering

\begin{subfigure}{0.35\textwidth}\includegraphics[width=1.\linewidth]{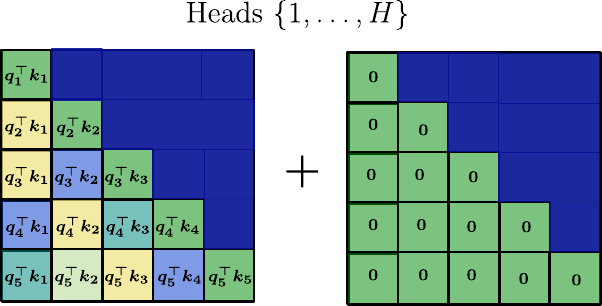}
 
 \caption{}\label{fig:nope}
\end{subfigure}
\hspace{2cm}
\begin{subfigure}{0.35\textwidth}
 \includegraphics[width=1.1\linewidth]{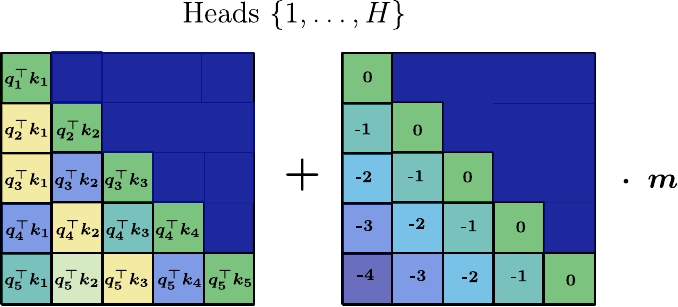}

 \caption{}\label{fig:alibi}
\end{subfigure}
\hfill

\begin{subfigure}{0.685\textwidth}

   \includegraphics[width=1.\linewidth]{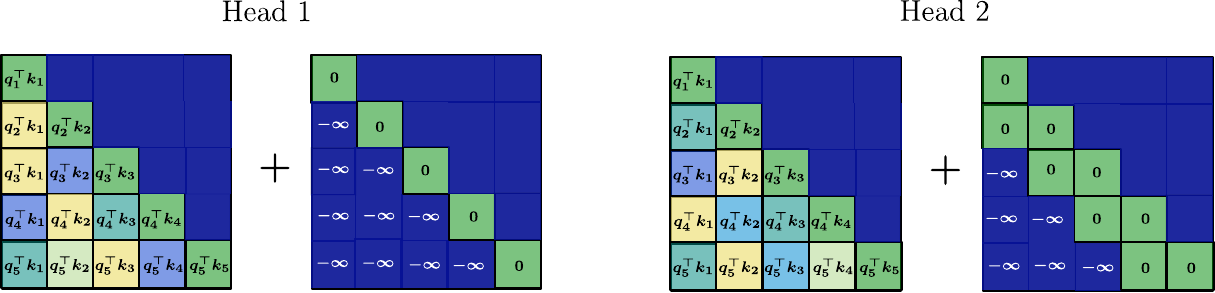}
   \caption*{}
\end{subfigure}\\
\begin{subfigure}{0.685\textwidth}

   \includegraphics[width=1.\linewidth]{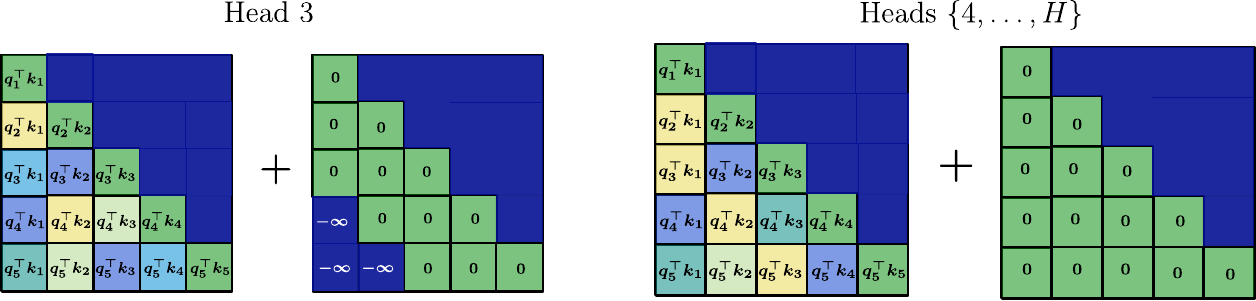}
   \caption{}\label{fig:hard_alibi}
\end{subfigure}
\caption{\small \textbf{Positional encoding schemes for transformers}: illustration of the different positional encodings of the transformers that are trained in our experiments. (a) corresponds to the NoPE encoding \cite{kazemnejad2023impact} where no positional encoding is applied to any of the attention heads (b) depicts the ALiBi encoding \cite{press2021train} where $m$ is a head-specific scalar and (c) the Hard-ALiBi encoding introduced in \autoref{sec:theory}. For the sake of illustration, we consider the case where we mask three heads which means that we force Heads 1, 2 and 3 to attend to their current token, their current and preceding tokens and their current, preceding and prior to the preceding tokens. The remaining heads are set as NoPE heads.}
\label{fig:hard-alibi-drawing}
\end{figure}

 We consider multiple positional encoding schemes in our experiments in \autoref{sec:learning}:

\begin{itemize}
\item[--] the NoPE scheme \cite{kazemnejad2023impact} where no positional information is added to any of the attention scores (\autoref{fig:nope}). This architecture choice helps to get better length generalization in multiple tasks including the copy task.
\item[--] the ALiBi scheme \cite{press2021train} which biases the attention scores with a penalty
that is proportional to their distance (\autoref{fig:alibi}). $m$ is a head-specific slope fixed before training.
\item[--] the Hard-ALiBi scheme introduced in \autoref{sec:theory} which has $M$ masked attention heads where we explicitly force the model to attend to their directly previous tokens and $H-M$ heads set to be NoPE attention heads. In \autoref{fig:hard_alibi}, we display the case where we have $M=4$ masked heads: in the first head, the tokens just attend to themselves; in the second head, the tokens attend to themselves and to previous ones; in the third head, the tokens attend to themselves, the previous ones and the second preceding tokens. The remaining $H-M$ heads are set to NoPE.
\end{itemize}

\subsection{Pretraining and evaluation details}\label{app:train_eval_details}

\paragraph{Software dependencies.} We implement all of our training in Pytorch \citep{paszke2019pytorch}. We use the HuggingFace library \citep{wolf2019huggingface} and the Mamba GitHub repository \citep{gu2023mamba}.  

\paragraph{Architectures.} In our experiments in \autoref{sec:learning}, the backbone of our transformers is the GPT-NeoX architecture. We set the number of layers to 12, the hidden size to 1024 and the number of heads $H=16$. We consider the different positional encodings that are described in \autoref{app:transformers}. For Alibi, we set the head-specific scalar as in the original paper i.e.\ $m_h=2^{-h/2}$ for $h\in\{1,\dots,H\}.$ For the Hard-Alibi model, we sweep over the number of masked heads $M\in \{2,\dots,10\}$ and found that the best model corresponds to $M=6$. Regarding the Mamba models, we set the number of layers to 24 and the hidden size 1024. We also sweep over the state space dimension $S\in\{16,32,64,128,256\}$ and found the best model is $S=32$. This choice of hyperparameters ensures that both transformers and Mamba models have a comparable number of parameters. Lastly, our LSTM is made of 4 layers and width 1024.

\paragraph{Training hyperparameters.} In \autoref{sec:learning}, at each epoch, we sample online a batch size of size 64. We fill the context with examples so we choose a context length ($C=420$ for all the experiments except \autoref{fig:in_distribution_copy} where we set $C=620$)  and pack as many examples as possible to fit this context. So in our case, one sample contains many instances. We run the experiments for 15 epochs for both transformers and Mamba while for LSTMs we need 300 epochs. All methods are trained with the AdamW optimizer \cite{loshchilov2017decoupled} with learning rate 5e-5, a linear rate decay schedule, 300 steps of warmup and default weight decay of 1e-1. For LSTMs and Mamba, we did a sweep over learning rates $\{5\mathrm{e}-5,1\mathrm{e}-4,5\mathrm{e}-4\}.$ Finally, to train all the models, we use the next-token prediction loss but we apply a mask on the input instance so that we only penalize the model whenever it makes a mistake on the labels (and not on the inputs and labels jointly).

\paragraph{Compute resources.} Pretraining was all done on an internal cluster using RTX8000 GPUs. We estimate that the final training run needed to produce the results in the paper took approximately 600 GPU hours. 

\paragraph{Evaluation algorithm.} We evaluate the models over 10 batches of size 64 for all the tasks except for the question answering one where we evaluate over 50 questions because the number of questions with a given context length is limited.

\paragraph{Decoding algorithm.} At inference, all our models use greedy decoding for generation and we set the temperature to 0.

\section{Additional Experiments}

 In \autoref{app:data_efficiency}, we focus on the in-distribution learning of the copy task and show that the number of samples needed by GSSMs is much higher than the one for transformers. In \autoref{app:pretrained_copy_uniform}, we study the performance of pre-trained models on the copy task in the case where the strings are sampled uniformly. This experiment shows that when the text to copy is totally random, the gap between pre-trained transformers and GSSMs is even larger.

\subsection{Data efficiency on the copy task}\label{app:data_efficiency}

\begin{figure*}[h!]
\centering
\begin{subfigure}[h!]{0.31\textwidth}
\centering
\vskip 0pt
    \includegraphics[height=3.8cm]{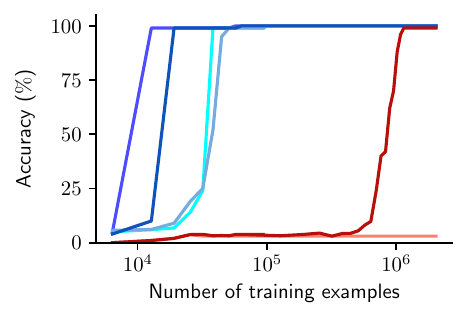}
    \caption{}
    \label{fig:char_level_id}
\end{subfigure}
\hfill
\begin{subfigure}[h!]
{0.31\textwidth}
\centering
\vskip 0pt
\includegraphics[height=3.8cm]{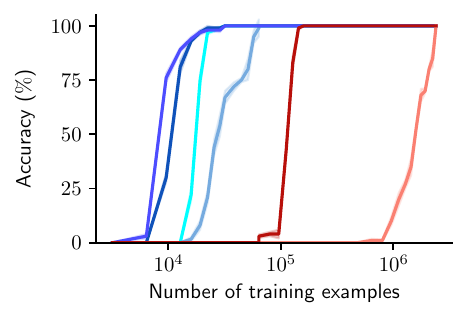}
\caption{}
\label{fig:str_level_short}
\end{subfigure}
\hfill
\begin{subfigure}[h!]
{0.31\textwidth}
\centering
\vskip 0pt
\includegraphics[height=3.8cm]{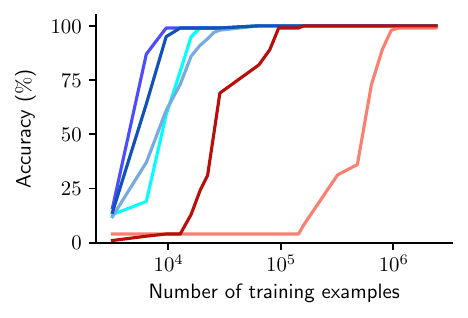}
\caption{}
\label{fig:char_level_short}
\end{subfigure}

\vspace{.3cm}

\small 
Transformer: \cblock{0}{255}{255} \hspace{0.5mm}RoPE\hspace{1mm} \cblock{118}{171}{223} \hspace{0.5mm}NoPE\hspace{1mm} \cblock{15}{82}{186}\hspace{1mm}Alibi\hspace{1mm} \cblock{77}{77}{255}\hspace{1mm}HAlibi\\ GSSM:
\cblock{250}{128}{114}\hspace{1mm}LSTM\hspace{1.5mm}\cblock{184}{15}{10}\hspace{1mm}Mamba
\caption{\textbf{(a) Copying long strings: character-level accuracy.} Here we train models to copy strings of length $ \leq 300$ and evaluate character-level accuracy on strings of length 300. Transformers train much faster than GSSMs. Mamba has a more progressive learning curve than in \autoref{fig:in_distribution_copy}. An LSTM cannot even learn the task within this number of samples at the character level. 
\textbf{(b) Copying short strings: string-level accuracy.} Here we train models to copy strings of length $ \leq 30$ and evaluate character-level accuracy on strings of length 30. Transformers train much faster than GSSMs. Compared to \autoref{fig:in_distribution_copy}, we see that Mamba needs way less samples in order to learn to copy length-30 strings. An LSTM can learn to copy but requires 100x more  training examples.
\textbf{(c) Copying short strings: character-level accuracy.}  Here we train models to copy strings of length $ \leq 30$ and evaluate character-level accuracy on strings of length 30 and report the character-level accuracy.}
\end{figure*}

In this section, we provide additional plots to complement the data efficiency experiment from \autoref{fig:in_distribution_copy}. We want to highlight the following points: 
\begin{itemize}
\item[--] in \autoref{fig:in_distribution_copy}, we  see a sharp transition for the Mamba learning curve. However, \autoref{fig:char_level_id} shows that the learning process is more smooth at the character level. Besides, LSTMs are not able to learn the copy on length-300 strings even at the character level. 
\item[--] We consider the experiment of learning to copy much shorter strings namely strings with length $\leq 30$. \autoref{fig:str_level_short} shows that the gap in terms of training examples between transformers and Mamba is much smaller i.e.\ Mamba only needs 10x more data. Besides, we see that the LSTM is able to learn the copy task but it needs 100x more data than transformers.
\end{itemize}

\subsection{Pre-trained models on the uniform copy task}\label{app:pretrained_copy_uniform}

In this section, we provide an additional experiment that shows the superiority of pre-trained Pythia over pre-trained Mamba models in the copy task.

\begin{wrapfigure}[14]{r}{0.45\textwidth}

\vspace{-.7cm}

\includegraphics[height=4.2cm]{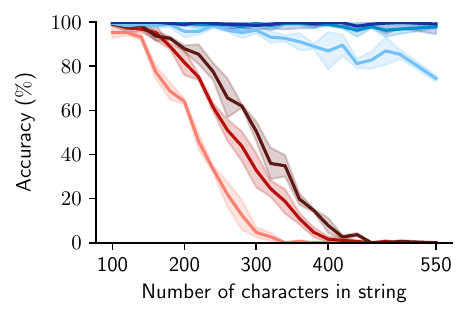}
\\
\centering
\small
Pythia: \cblock{115}{194}{251}\hspace{1mm}410M
\hspace{1mm}\cblock{0}{142}{204}\hspace{1mm}1.4B
\hspace{1mm}\cblock{17}{52}{166}\hspace{1mm}2.8B\\
Mamba:\hspace{1mm}\cblock{250}{128}{114}\hspace{1mm}360M
\hspace*{1mm}\cblock{184}{15}{10}\hspace{1mm}1.4B\hspace{1mm}\cblock{94}{25}{20}\hspace{1mm}2.8B
\caption{\small \textbf{Copy: uniform strings.} To test whether it mattered that the strings were in natural language, we generate uniformly sampled strings (the generation process is described in \autoref{sec:learning}). We find that this degrades the Mamba models while Pythia models are able to keep a high performance. }\label{fig:pretrained} 
\end{wrapfigure}

We consider the same setup as in \autoref{sec:learning}: we sample uniform strings of alphabet characters with a fixed length and ask the model to copy it by using the same prompt format as the one described in \autoref{subsec:copy_pretrained}.

This setting is a more extreme version of \autoref{fig:copy_shuffle} since the strings are more random: in  \autoref{fig:copy_shuffle}, the order of the nouns were random but the nouns were English nouns while in \autoref{fig:copy_shuffle}, the strings are totally random. In \autoref{fig:pretrained}, we see a clear separation between the transformers and Mamba models with the smallest Pythia outperforming the largest Mamba. However, compared to \autoref{fig:copy_shuffle}, the Pythia performance is much higher since the 1.4B model able to get almost 100\% accuracy.

\vspace{1cm}

\subsection{Performance $n$-gram models at copying}

\begin{wrapfigure}[14]{r}{0.45\textwidth}

\vspace*{-1.3cm}
    \includegraphics[height=4.2cm]{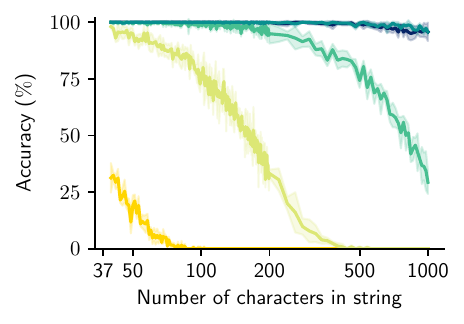}\\
    \centering
    \small
    n-gram length: {\small \cblock{255}{213}{0} \hspace{1mm}2\hspace{1.5mm}\cblock{220}{231}{117} \hspace{1mm}3\hspace{1.5mm} \cblock{72}{191}{145}\hspace{1mm}4\hspace{1.5mm}\cblock{8}{143}{143}\hspace{1mm}5}\\
    \hspace{-.5cm}Transformers: {\small \cblock{8}{37}{103} \hspace{1mm}Hard-ALiBi}
    \vspace{-0.2cm}
    \caption{String-level copying accuracy obtained by perfect $n$-gram models and Transformers with Hard-ALiBi. Transformers performance matches the one of 5-gram model.}
    \label{fig:n_gram_perf}
    \vspace{-0.3cm}
\end{wrapfigure}

In \autoref{fig:n_gram_perf}, we display the performance of perfect $n$-gram models in the copy task. To obtain these curves, we uniformly sample 128 strings over 3 seeds and report the probability there is a $n$-gram. This probability corresponds to the performance of a perfect $n$-gram model. We observe that Transformers enhanced with the Hard-ALiBi positional encoding have a performance close to a perfect $5$-gram model.

\newpage
\section{Proofs - Upper Bound}

This section gives a detailed proof of Theorem \ref{thm:upper_bound} and Lemma \ref{lem:qgram}.

\subsection{Technical Lemmas}
We begin by introducing some technical lemmas that we use in the proof of Theorem \ref{thm:upper_bound}.

\begin{algorithm}[tb]
   \caption{Hash-based copying}
   \label{alg:copy}
\begin{algorithmic}
   \STATE {\bfseries Input:} sequence $x_1, \dots, x_L$
   \STATE Let $s: \sD^* \to \reals^d$ be some hashing function.
   \FOR{$i=n+2, \dots, L$}
   \STATE $k_i \leftarrow s(x_{i-n}, x_{i-n+1}, \dots, x_{i-1})$
   \STATE $v_i \leftarrow x_i$
   \ENDFOR
   \FOR{$j=1, \dots, L$}
   \IF{$j \le n$}
   \STATE $y_j \leftarrow x_j$
   \ELSE
   \STATE $q_j \leftarrow s(y_{j-n}, \dots, y_{j-1})$
   \STATE Let $i \in [L]$ s.t. $k_i = q_j$, and set $y_j \leftarrow x_i$
   \ENDIF
   \ENDFOR
   \STATE {\bfseries Output:} sequence $y_1, \dots, y_L$
\end{algorithmic}
\end{algorithm}

\begin{lemma}
    \label{lem:hard_alibi}
    Let $h_t(\vx_1, \dots, \vx_i) = \frac{1}{\min(t, i)}\sum_{j=\max(1,i-t+1)}^i \vx_j$. Then, $h_t$ can be computed using a hard-ALiBi attention head.
\end{lemma}

\begin{proof}
    Let $W_k, W_q = 0$ (zero matrix) and let $W_v = I_d$ (indentity matrix). We choose $b_i \in \{0, -\infty\}^i$ s.t. $$b_{i,j} = \begin{cases}
        -\infty & j \le i-t \\
        0 & j > i-t
    \end{cases}$$
    
\end{proof}

\begin{lemma}
    \label{lem:embedding}
Assume that $d = \left \lceil \log(D)\right \rceil+2$. Then, there exists an embedding $\Psi$ s.t. 
\begin{itemize}
    \item For every $x \in \sD$ it holds that $\norm{\Psi(x)}_2 = 1$ and $\norm{\Psi(x)}_\infty \le 1$.
    \item For $x' \ne x$ it holds that $\inner{x,x'} < 1-\frac{1}{d}$.
    \item For every $x \ne \BOS$, $\inner{\Psi(x), \Psi(\BOS)} = 0$, and for every $x \ne \COPY$, $\inner{\Psi(x), \Psi(\COPY)} = 0$.
\end{itemize}
\end{lemma}

\begin{proof}
    Denote $d' = \left \lceil \log(D)\right \rceil$, and observe that we can encode all $D$ ``non-special'' tokens as vectors in $\left\lbrace\pm \frac{1}{\sqrt{d}}\right\rbrace^{d'}$, and denote this encoding by $\Psi'$. Now, define:
    \[
    \Psi(x) = \begin{cases}
        [1, 0, \dots, 0] & x = \BOS \\
        [0,1,0, \dots, 0] & x = \COPY \\
        [0, 0, \Psi'(x)] & o.w.
    \end{cases}
    \]
\end{proof}

\begin{lemma}
    \label{lem:softmax}
    Let $\vz \in \reals^K$ be some vector such that, for some constants $a > b > 0$, there exists $i \in [K]$ s.t. $z_i = a$ and for all $j \ne i$ we have $\abs{z_j} \le b$. Denote $\vs = \softmax(\vz)$. Then $s_i \ge \frac{1}{1+K\exp(b-a)}$ and $s_j \le \exp(b-a)$ for all $j \ne i$.
\end{lemma}

\begin{proof}
    First, notice that:
    \begin{align*}
        \exp(a) = \exp(z_i) \le \sum_{j=1}^K \exp(z_j) \le \exp(z_i) + (K-1) \exp(b) \le \exp(a) + K \exp(b) = \exp(a) (1 + K \exp(b-a))
    \end{align*}
    
    Observe the following:
    \begin{align*}
        s_i = \frac{\exp(z_i)}{\sum_{j=1}^K\exp(z_j)} \ge \frac{\exp(a)}{\exp(a)(1+K\exp(b-a))} = \frac{1}{1+K\exp(b-a)}
    \end{align*}
    Finally, for every $j \ne i$:
    \begin{align*}
        s_j = \frac{\exp(z_j)}{\sum_{j=1}^K \exp(z_j)} \le \frac{\exp(b)}{\exp(a)} = \exp(b-a)
    \end{align*}
\end{proof}

\subsection{Proof of Theorem \ref{thm:upper_bound}}

We begin by constructing the first block of the transformer, which computes the ``lookup-table'' for the copy algorithm. This lookup-table consists of pairs of (key,values) for each position $i$, where the key encodes the $n$-gram preceding the $i$-th token, and the value is the $i$-th token. Namely, if the sequence is $x_1, \dots, x_i$, then $\mathrm{key}_i = (x_{i-n-1}, \dots, x_i)$ and $\mathrm{value}_i = x_i$. Additionally, the transformer block also computes a query, which is just the ``current'' $n$-gram, i.e. $\mathrm{query}_i = (x_{i-n}, \dots, x_i)$. The copy algorithm matches the current $\mathrm{query}$ with previous $\mathrm{key}$-s, retrieving the matching $\mathrm{value}$.

The following theorem shows that by using a combination of $n$ hard-ALiBi attention heads (with different choice of $m$ for each head), together with an MLP layer, can compute the correct $(\mathrm{key}_i,\mathrm{value}_i,\mathrm{query}_i)$ for each position. We use a slightly modified $\mathrm{key}_i,\mathrm{query}_i$ to handle cases where the $i \le n$ (or, $i$ is one of the first $n$ tokens after the $\COPY$ token). 

\begin{lemma}
    \label{lem:first_block}
    Let $\Psi$ be the one-hot embedding. Then, there exists a hard-ALiBi transformer block with 3 outputs, denoted $T^{\key}, T^{\query}, T^{\val}$, which correspond to 3 blocks of the output dimension, s.t. $T^{\key} : \reals^{d \times *} \to \reals^{(d+1)n \times *}$, $T^{\query} : \reals^{d \times *} \to \reals^{(d+1)n \times  *}$ and $T^{\val} : \reals^{d \times *} \to \reals^{d \times *}$ satisfying, for all $\vx$ sampled from a length-$n$ copy distribution,
    \begin{enumerate}
        \item Value output: for all $i$, $$T_i^\val(\Psi(x_1), \dots, \Psi(x_i)) = \Psi(x_i)$$
        \item Key output:
        \begin{itemize}
        \item For $t = 1, \dots, n$, if $i > n$
        \[
        T^{\key}_{(t-1)d+1:td,i}(\Psi(x_1), \dots, \Psi(x_i)) = \Psi(x_{i-t})
        \]
        and if $i \le n$
        \[
        T^{\key}_{(t-1)d+1:td,i}(\Psi(x_1), \dots, \Psi(x_i)) = 0
        \]
        \item Additionally, for $t=1, \dots, n$, for all $i$
        \[
        T^{\key}_{nd+t,i}(\Psi(x_1), \dots, \Psi(x_i)) = \1\{i=t+1\}
        \]
        \end{itemize}
    \item Query output:
    \begin{itemize}
        \item For $t = 1, \dots, n$, if $i \ge n$
        \[
        T^{\query}_{(t-1)d+1 : td,i}(\Psi(x_1), \dots, \Psi(x_i)) = \Psi(x_{i-t+1})
        \]
        and if $i < n$
        \[
        T^{\query}_{(t-1)d+1 : td,i}(\Psi(x_1), \dots, \Psi(x_i)) = 0
        \]
        \item Additionally, for $t = 1, \dots, n$, for all $i$
        \[
        T^{\key}_{nd+t,i}(\Psi(x_1), \dots, \Psi(x_i)) = n \cdot \1\{i=L+t\}
        \]
    \end{itemize}
    \end{enumerate}
    
\end{lemma}

\begin{proof} We prove the following:

\begin{enumerate}
    \item For the value output, we simply take $T^{\val} = h_1$ as defined in Lemma \ref{lem:hard_alibi}.
    \item For each $t =0, \dots, n$, define:
    \[
    g_t(\vx_1, \dots, \vx_i) = (t+1) \cdot h_{t+1}(\vx_1, \dots, \vx_i) - t \cdot h_{t}(\vx_1, \dots, \vx_i)
    \]
    where we define $h_0 \equiv 0$.
    Observe that if $i > t$ then:
    \[
    g_{t}(\vx_1, \dots, \vx_i) = (t+1)\cdot \frac{1}{t+1} \sum_{j=i-t}^i\vx_j - t \cdot \frac{1}{t}\sum_{j=i-t+1}^i\vx_j = \vx_{i-t}
    \]
    and if $i \le t$ then:
    \[
    g_{t}(\vx_1, \dots, \vx_i) = t \cdot \frac{1}{i} \sum_{j=1}^i \vx_j - (t-1) \cdot \frac{1}{i} \sum_{j=1}^i \vx_j = \frac{1}{i}\sum_{j=1}^i \vx_j
    \]
    For every $j \in [d]$, denote 
    \begin{align*}
    \hat{g}_{t,j}(\vx_1, \dots, \vx_i) &= \sigma(e_j \cdot g_t(\vx_1, \dots, \vx_i) - n \Psi(\BOS) \cdot g_n(\vx_1, \dots, \vx_i)) \\
    &-\sigma(-e_j \cdot g_t(\vx_1, \dots, \vx_i) - n \Psi(\BOS) \cdot g_n(\vx_1, \dots, \vx_i))
    \end{align*}
    \textbf{Claim:} $\hat{g}_t(\Psi(x_1), \dots, \Psi(x_i)) = \1 \{i > n\} \cdot \Psi(x_{i-t})$
    
    \textbf{Proof:} Fix some $j \in [d]$. Observe that for all $i$,  $\abs{e_j \cdot g_t(\Psi(x_1), \dots, \Psi(x_i))} \le 1$.
    \begin{itemize}
        \item If $i \le n$, we have $g_n(\Psi(x_1), \dots, \Psi(x_i)) = \frac{1}{i} \sum_{j'=1}^i \Psi(x_{j'})$ and so $\Psi(\BOS) \cdot g_n(\Psi(x_1), \dots, \Psi(x_i)) = 1$ where we use the properties of $\Psi$ and the fact that $x_1 = \BOS$. Therefore, $\hat{g}_{t,j}(\Psi(x_1), \dots, \Psi(x_i)) = 0$.
        \item If $i > n \ge t$, then:
        \begin{align*}
        \hat{g}_{t,j}(\Psi(x_1), \dots, \Psi(x_i)) &= \sigma \left(e_j \cdot \Psi(x_{i-t}) - n \Psi(\BOS) \cdot \Psi(x_{i-t})\right) \\
        &- \sigma \left(-e_j \cdot \Psi(x_{i-t}) - n \Psi(\BOS) \cdot \Psi(x_{i-t})\right) \\
        &= \sigma(e_j \cdot \Psi(x_{i-t})) - \sigma(-e_j \cdot \Psi(x_{i-t})) = e_j \cdot \Psi(x_{i-t})
        \end{align*}
        where we use the fact that $x_{i-t} \ne \BOS$ and therefore $\Psi(\BOS) \cdot \Psi(x_{i-t}) = 0$.
    \end{itemize}

    Denote $$\tilde{g}_t(\vx_1, \dots, \vx_i) = \frac{1}{2}\sigma\left(2\Psi(\BOS) \cdot \left(g_t(\vx_1, \dots, \vx_i)-h_1(\vx_1, \dots, \vx_i)\right)-1\right)$$
    
    \textbf{Claim:} $\tilde{g}_t(\Psi(x_1), \dots, \Psi(x_i)) = \1\{i = t+1\}$

    \textbf{Proof:} Denote $g_{t,i} = g_t(\Psi(x_1), \dots, \Psi(x_i))$ and $h_{1,i} = h_1(\Psi(x_1), \dots, \Psi(x_i))$. Observe:
    \begin{itemize}
        \item If $i = t+1$, then $g_{t,i} = \Psi(x_1) = \Psi(\BOS)$ and $h_{1,i} = \Psi(x_i) \perp \Psi(\BOS)$ and therefore $\tilde{g}_{t,i} = 1$.
        \item If $i > t+1$ then $g_{t,i} = \Psi(x_{i-t}) \perp \Psi(\BOS)$ and $h_{1,i} = \Psi(x_i) \perp \Psi(\BOS)$ and so $\tilde{g}_{t,i} = 0$.
        \item If $1 < i \le t$ then $\Psi(\BOS) \cdot g_{t,i} = \frac{1}{i} \le \frac{1}{2}$ and $h_{1,i} = \Psi(x_i) \perp \Psi(\BOS)$ and so $\tilde{g}_{t,i} = 0$.
        \item If $i = 1$ then $g_{t,i} = h_{1,i} = \Psi(\BOS)$ and therefore $\tilde{g}_{t,i} = 0$.
    \end{itemize}
     Finally, we can take $T^{\key} = [\hat{g}_1, \dots, \hat{g}_q, \tilde{g}_1, \dots, \tilde{g}_q]$.
     
     \item For all $t = 1, \dots, n$, define $g^*_t(\vx_1, \dots, \vx_i) = \sigma(\Psi(\COPY) \cdot g_{t-1}(\vx_1, \dots, \vx_i))$.

     \textbf{Claim:} $g^*_t(\Psi(x_1), \dots, \Psi(x_i)) = \1\{i =L+t\}$

     \textbf{Proof:} Denote $g_{t,i} = g_t(\Psi(x_1), \dots, \Psi(x_i))$. Observe:
     \begin{itemize}
         \item If $i = L + t$ then $g_{t-1,i} = \Psi(x_{i-t+1}) = \Psi(x_{L+1}) = \Psi(\COPY)$ and therefore $g^*_{t,i} = 1$.
         \item If $i \ne L+t$ and $i > t-1$ then $g_{t-1,i} = \Psi(x_{i-t+1}) \perp \Psi(\COPY)$ and therefore $g_{t,i}^* = 0$.
         \item If $i \le t$ then since $x_1, \dots, x_i \ne \COPY$ we get $\Psi(\COPY) \cdot g_{t-1,i} = 0$ and therefore $g_{t,i}^* = 0$.
     \end{itemize}
    Therefore, we can take $T^{\mathrm{query}} = [\hat{g}_0, \dots, \hat{g}_{q-1}, n \cdot g^*_{1}, \dots, n \cdot g^*_q]$.
\end{enumerate}

\end{proof}

Now, we prove Theorem \ref{thm:upper_bound} by showing that using a single attention head with no positional embedding on top of the construction in Lemma \ref{lem:first_block} realizes the copy algorithm. Since the first block computes the correct choice of $\mathrm{key}_i,\mathrm{query}_i,\mathrm{value}_i$, by correctly scaling of the attention matrix we verify that the output of the second layer at position $i$ corresponds to $\approx \mathrm{value}_j$ for $j$ s.t. $\mathrm{key}_j = \mathrm{query}_i$.

\begin{proof}[Proof of Theorem \ref{thm:upper_bound}]
    Let $T^{\val}, T^{\key}, T^{\query}$ be the outputs of the Transformer block guaranteed by Lemma \ref{lem:first_block}. Observe that, for some temprature $\tau \in \reals$, the following function can be computed by a softmax-attention layer on-top of this block:
    \[
    H(\Psi(x_1), \dots, \Psi(x_i)) = T^{\val} \cdot \softmax(\tau \cdot T^\key \cdot T^\query_i)
    \]
    where e.g. $T^\val$ denotes  $ T^\val(\Psi(x_1), \dots, \Psi(x_i))$.

    For now, assume that all the $n$-grams in $\vx$ are unique, and that the length of the input satisfies $2L+2 \le K$ for $K = D^n$.

    \textbf{Claim:} Fix some $i > L$, denote $\vz = T^\key \cdot T^\query_i$. Then, $z_{i-L+1} = n$ and $\abs{z_j} < n-\frac{1}{d}$ for all $j \ne i-L+1$.
    
    \textbf{Proof:} We separate to the following cases:
    \begin{itemize}
        \item If $i > L+n-1$, then for every $j$ we have 
        \begin{align*}
            T^\key_j \cdot T^\query_i &= \1\{j > n\} \cdot [\Psi(x_{j-1}), \dots, \Psi(x_{j-n})]^\top [\Psi(x_i), \dots, \Psi(x_{i-n+1})] \\
            &= \1\{j > n\} \cdot \sum_{t=1}^n \Psi(x_{j-t}) \Psi(x_{i-t+1})
        \end{align*}
        Now, if $j = i-L+1$ then $x_{j-t} = x_{i-L+1-t} = x_{i-t+1}$ and since $j > n$ we get $$T_j^\key \cdot T_i^\query = \sum_{t=1}^n \norm{\Psi(x_{i-t+1})}=n$$
        If $j \ne i - L +1$, since there are no repeated $n$-grams, there is at least some $t \in [n]$ s.t. $\Psi(x_{j-t}) \ne \Psi(x_{i-t+1})$ and by the choice of the embedding $\Psi(x_{j-t}) \cdot \Psi(x_{i-t+1}) \le 1-\frac{1}{d}$. In this case, we get $\abs{T_j^\key \cdot T_i^\query} \le n-\frac{1}{d}$.
        \item If $L < i \le L+n-1$ and $j \le n$ then
        \[
        T_j^\key \cdot T_i^\query = n e_{j-1} \cdot e_{i-L} = n \cdot \1\{j= i-L+1\}
        \]
        which satisfies the required.
        \item If $L < i \le L +n -1$ and $j > n$ then
        \[
        T_j^\key \cdot T_i^\query = \sum_{t=1}^n \Psi(x_{j-t}) \Psi(x_{i-t+1})
        \]
        and as before, since there are no repeated $n$-grams, we get $\abs{T_j^\key \cdot T_i^\query} \le n-\frac{1}{d}$
    \end{itemize}

    \textbf{Claim:} Fix some $\epsilon \in (0,1)$ and some $i > L$, denote $\vs = \softmax(\tau T^\key \cdot T^\query_i) = \softmax(\tau \cdot \vz)$. If $\tau = d\ln(\frac{2K}{\epsilon})$, then $s_{i-L+1} \ge 1-\epsilon$ and $s_j \le \frac{\epsilon}{2K}$ for all $j \ne i-L+1$.

    \textbf{Proof:} Using the previous claim, together with Lemma \ref{lem:softmax}, we get that:
    \begin{itemize}
        \item $s_{i-L+1} \ge \frac{1}{1+i \exp(-\tau/d)} \ge \frac{1}{1+K \exp(-\tau/d)} \ge \frac{1}{1+\epsilon/2} = 1- \frac{\epsilon/2}{1+\epsilon/2}\ge 1-\epsilon$
        \item For $j \ne i-L+1$, 
        \[
        s_j \le \exp(-\tau/d) \le \frac{\epsilon}{2K}
        \]
    \end{itemize}

    \textbf{Claim:} Fix some $\epsilon \in (0,1)$ and some $i > L$. Then, for $\tau \ge d\ln(\frac{2K}{\epsilon})$, it holds that: $$\norm{H(\Psi(x_1), \dots, \Psi(x_i))-\Psi(x_{i-L+1})} \le \epsilon$$

    \textbf{Proof:} Let $\vs$ as defined in the previous claim. Then:
    \begin{align*}
        \norm{H(\Psi(x_1), \dots, \Psi(x_i))-\Psi(x_{i-L+1})}
        &= \norm{\sum_{j=1}^i s_j\Psi(x_{j})-\Psi(x_{i-L+1})} \\
        &\le (1-s_{i-L+1})\norm{\Psi(x_{i-L+1})} + \sum_{j \ne i-L+1}s_j \norm{\Psi(x_j)} \\
        &= (1-s_{i-L+1}) + \sum_{j \ne i - L+1}s_j \le \epsilon + (i-1)\frac{\epsilon}{2K} \le 2 \epsilon
    \end{align*}

    Now, denote by $\Phi : \reals^d \to \sD$ the output map given by $\Phi(\vz) = \arg \max_{x \in \sD} \vz \cdot \Psi(x)$ (which can be computed by an $\argmax$ over a linear function).

    \textbf{Claim:} If $\tau \ge d \ln(8Kd)$, then for all $i>L$ we have $\Phi (H(\Psi(x_1), \dots, \Psi(x_i))) = x_{i-L+1}$.
    
    \textbf{Proof:} Denote $\vy_i = H(\Psi(x_1), \dots, \Psi(x_i))$. 
    First, using the previous claim, we observe that
    \begin{align*}
        \vy_i \cdot \Psi(x_{i-L+1}) &= (\vy_i-\Psi(x_{i-L+1})) \cdot \Psi(x_{i-L+1}) + \norm{\Psi(x_{i-L+1})} \\
        &\ge 1-\norm{\vy_i-\Psi(x_{i-L+1})} \ge 1-\frac{1}{4d}
    \end{align*}
    Next, observe that for all $j \ne i-L+1$ we have
    \begin{align*}
        \vy_i \cdot \Psi(x_j) &= (\vy_i-\Psi(x_{i-L+1}))\cdot \Psi(x_j) + \Psi(x_j) \cdot \Psi(x_{i-L+1}) \\
        &\le \norm{\vy_i - \Psi(x_{i-L+1})} +1 - \frac{1}{d} \le 1-\frac{3}{4d} < \vy_i \cdot \Psi(x_{i-L+1})
    \end{align*}

    From the above claim, the Transformer construction outputs the correct token at each step of the auto-regressive generation.
\end{proof}

\subsection{Proof of Lemma \ref{lem:qgram}}

\begin{proof}[Proof of Lemma \ref{lem:qgram}]
    Fix some $i < j \in [L]$. Let $I := \{i, \dots, i+n\}$ and $J := \{j, \dots, j+n\}$. We first bound the probability of drawing some $\vx$ s.t. $\vx_I = \vx_J$.
    Note that there are $D^{\abs{I \cup J}}$ choices for $\vx_{I \cup J}$.
    We count the number of choices for $\vx_{I \cup J}$ s.t. $\vx_I = \vx_J$. Notice that in this case, $\vx_{I \cup J}$ is determined by $\vx_{I \setminus J}$, therefore there are $D^{\abs{I \setminus J}}$ possible choices. We conclude that
    \[
        \Pr\left[\vx_I = \vx_J\right] = \frac{D^{\abs{I \setminus J}}}{D^{\abs{I \cup J}}} = D^{\abs{I \setminus J}- \abs{I \cup J}} = D^{-n}
    \]
    Using the union bound, we get that
    \[
        \Pr\left[\exists i < j ~\mathrm{s.t.}~\vx_{i, \dots, i+n} = \vx_{j, \dots, j+n}\right]\le \sum_{i<j} \Pr\left[\vx_{i, \dots, i+n} = \vx_{j, \dots, j+n}\right]< L^2 D^{-n}
    \]
\end{proof}

\section{Proofs - Lower Bound}

In this section, we prove Theorem \ref{thm:lower_bound}. We begin by showing that, for every input, the output of the model in each iteration is a deterministic function of the state of the model after observing the input:

\begin{lemma}
    \label{lem:fixed_state}
    Let $H_{u,r} : \sD^{n'} \to \sD^{n}$ be some fixed-state sequence-to-sequence model. Then, there exists map $G : \gS \to \sD^n$ s.t. for all $\vx \in \sD^{n'}$
    \[
        H_{u,r}(\vx) = G \circ S_{n'} (\vx)
    \]
\end{lemma}

\begin{proof}
    Let $x_{n'+1}, \dots, x_{n'+n}$ be the outputs of $H_{u,r}$. We need to show that there exist functions $G_1, \dots, G_n$ s.t. $H_{u,r}(x_1, \dots, x_{n'}) = G(S_{n'}(x_1, \dots, x_n))$.
    We give the following recursive definition:
    \begin{itemize}
        \item $G_1(s) = r(s)$, $\tilde{G}_1(s) = u(s, G_1(s))$.
        \item $G_i(s) = r(\tilde{G}_{i-1}(s))$, $\tilde{G}_i(s) = u(\tilde{G}_{i-1}(s), G_i(s))$.
    \end{itemize}
    Denote $s = S_{n'}(x_1, \dots, x_{n'})$
    We prove by induction that $G_i(s) = x_{n'+i}$ and also that $\tilde{G}_i(s) = S_{n'+i}(x_1, \dots, x_{n'+i})$.
    \begin{itemize}
        \item $G_1(s) = r(s) = R_{n'}(x_1, \dots, x_{n'}) = x_{n'+1} $.
        \item $\tilde{G}_1(s) = u(s, G_1(s)) = u(s, x_{n'+1}) = S_{n'+1}(x_1, \dots, x_{n'+1})$
        \item $G_i(s) = r(\tilde{G}_{i-1}(s)) = r(S_{n'+i-1}(x_1, \dots, x_{n'+i-1})) = R_{n'+i-1}(x_1, \dots, x_{n'+i-1}) = x_{n'+i}$
        \item $\tilde{G}_i(s) = u(\tilde{G}_{i-1}(s, G_i(s))) = u(S_{n'+i-1}(x_1, \dots, x_{n'+i-1}), x_{n'+i}) = S_{n'+i}(x_1, \dots, x_{n'+i})$
    \end{itemize}
    and so the required follows.
\end{proof}

Given the previous Lemma, we bound the error of the model by comparing the number of possible states to the number of possible inputs.

\begin{proof}[Proof of Theorem \ref{thm:lower_bound}]
    From Lemma \ref{lem:fixed_state}, there exists some function $G: \gS \to \sD^n$ s.t. $H_{u,r} = G \circ S_{n'}$. For each $\vx$, we denote by $\tilde{\vx}$ the sequence $\BOS, \vx, \COPY$. Now, observe the following:
    \begin{align*}
        1-\err_{\gD_n}(H_{u,r}) &= \Pr_{\gD_n} \left[H_{u,r}(\tilde{\vx}) = \vx\right] \\
        &= \frac{1}{D^n} \sum_{\vx \in \sD^n} \1\{H_{u,r}(\tilde{\vx}) = \vx\} \\
        &= \frac{1}{D^n} \sum_{s \in \gS} \sum_{\vx \in S_{n+2}^{-1}(\tilde{\vx})} \1\{G \circ S_{n'+2}(\tilde{\vx}) = \vx\} \\
        &= \frac{1}{D^n} \sum_{s \in \gS} \sum_{\vx \in S_{n+2}^{-1}(\tilde{\vx})} \1\{G(s) = \vx\} \le \frac{\abs{\gS}}{D^n}\\
    \end{align*}
\end{proof}

\end{document}